\documentclass{article}

\PassOptionsToPackage{table}{xcolor}

\usepackage[main, final]{neurips_2026}

\usepackage[utf8]{inputenc} 
\usepackage[T1]{fontenc}    
\usepackage{hyperref}       
\usepackage{url}            
\usepackage{booktabs}       
\usepackage{amsfonts}       
\usepackage{nicefrac}       
\usepackage{microtype}      
\usepackage{xcolor}         

\usepackage{graphicx}
\usepackage{subcaption}
\usepackage{amsmath}
\usepackage{amssymb}
\usepackage{mathtools}
\usepackage{amsthm}
\usepackage{tabularx}
\usepackage{makecell}
\usepackage{algorithm}
\usepackage{algorithmic}
\usepackage{enumitem}
\usepackage{siunitx}
\usepackage{wrapfig}
\usepackage{nicefrac}
\usepackage[capitalize,noabbrev]{cleveref}

\theoremstyle{plain}
\newtheorem{theorem}{Theorem}[section]

\theoremstyle{definition}
\newtheorem{definition}[theorem]{Definition}

\theoremstyle{remark}

\theoremstyle{definition}

\newcommand{\Tape}{\textsc{TAPE}}
\newcommand{\Entropy}{\mathsf{H}}
\newcommand{\IG}{\mathsf{IG}}
\newcommand{\MI}{\mathsf{I}}
\newcommand{\KL}{\mathrm{KL}}

\newcommand{\1}{\mathbf{1}}

\definecolor{oodbg}{gray}{0.93}
\definecolor{cival}{gray}{0.55}
\definecolor{bestval}{RGB}{0,92,175}
\definecolor{secondval}{RGB}{0,140,105}
\definecolor{checkbg}{gray}{0.96}
\definecolor{checkbd}{gray}{0.75}
\newcommand{\metricup}{\textcolor{cival}{$\uparrow$}}
\newcommand{\metricdown}{\textcolor{cival}{$\downarrow$}}
\newcommand{\best}[1]{\textcolor{bestval}{\textbf{#1}}}
\newcommand{\second}[1]{\textcolor{secondval}{\underline{#1}}}
\newcommand{\ci}[1]{\textcolor{cival}{\scriptsize{[#1]}}}
\title{\Tape{}: A Cellular Automata Benchmark for Evaluating Rule-Shift Generalization in Reinforcement Learning}

\author{%
  Enze Pan \\
  The University of Hong Kong \\
  \texttt{u3665478@connect.hku.hk}
}

\begin{document}
\maketitle

\begin{abstract}
Out-of-distribution generalization in reinforcement learning is hard to diagnose when benchmark shifts mix dynamics, observations, goals, and rewards. We address this with \Tape{}, a controlled benchmark that isolates \emph{latent rule-shift} in dynamics while keeping the observation-action interface fixed. The protocol combines deterministic splits, 20-seed replication, bootstrap uncertainty reporting, and continuous metrics for sparse-success regimes. Across baseline families, we find a consistent ID-to-OOD drop and strong heterogeneity across stable/periodic/chaotic rules. Importantly, this fragility appears even in an intentionally simple 1D deterministic setting, suggesting that many current RL algorithms remain brittle to latent-law changes under minimal confounds. To calibrate strict success, we report a protocol-matched true-dynamics random-shooting reference ($p_{\mathrm{oracle}}\approx0.187$) and oracle-normalized scores $\mathrm{ON}(p)=100\,p/p_{\mathrm{oracle}}$; this is a budgeted operational reference, not a global-optimality bound. A smaller feasibility regime ($L{=}H{=}16$) with 100\% rule-wise solvability helps separate reachability limits from policy failure. These results position \Tape{} as a mechanism-oriented diagnostic for robust adaptation and latent-mechanism inference, and as a controlled benchmark relevant to broader AGI-oriented evaluation without making strong AGI sufficiency claims.
\end{abstract}
\section{Introduction}
A persistent RL research gap is the mismatch between strong in-distribution (ID) optimization and reliable out-of-distribution (OOD) control when transition laws shift \citep{dulacarnold2021challenge,kirk2021survey,packer2018assessing}.
In many existing suites, OOD labels aggregate multiple perturbation sources (visual appearance, goals, dynamics coefficients, reward shaping), which weakens attribution from observed degradation to a specific mechanism \citep{cobbe2020procgen,kirk2021survey}.
Consider an agent in a discrete controlled system: it observes a 1D binary tape, applies local interventions, and the environment evolves deterministically.
The transition law is governed by an unobserved rule $z$ that maps local neighborhoods to the next tape configuration.
During training, $z$ is fixed within each episode but varies across episodes; during evaluation, rules are sampled from a disjoint holdout set.
The central research question is whether policy learning recovers transferable structure or overfits rule realizations seen during training.

We isolate \emph{rule shifts} in latent transition dynamics.
Specifically, \textbf{OOD} denotes evaluation on held-out rules $z\in\mathcal{Z}_{\mathrm{test}}$ with $\mathcal{Z}_{\mathrm{test}}\cap\mathcal{Z}_{\mathrm{train}}=\emptyset$, while the observation/action interface and goals remain fixed; this definition excludes unrelated forms of nonstationarity.
To study this, we introduce \Tape{}, a controlled environment derived from one-dimensional elementary cellular automata (CA) \citep{wolfram1983statmech,wolfram2002anks}.
In \Tape{}, each task is defined by a latent CA update rule, while the observation/action interface remains fixed.
This enables exact splits where only the transition rule changes between training and test tasks.
\Cref{fig:tape} shows a concrete step-by-step rollout under a fixed latent rule ($z{=}30$, $L{=}8$): each column is one time step (action then CA update).

This paper targets \textbf{benchmark construction and diagnostic inference} rather than algorithmic state-of-the-art claims.
The central design objective is to isolate one failure mechanism---generalization under latent rule shift---and to quantify it with a protocol that supports reproducible statistical inference.

Under a first-principles view, the task is a partially observed control process in which the latent rule $z$ governs transition kernels; effective control therefore requires both action optimization and belief refinement over $z$.
Existing RL families frequently optimize reactive behavior under fixed or weakly varying dynamics, yet do not consistently recover transferable latent-law inference under strict holdout rules \citep{rakelly2019pearl,chua2018pets,hafner2023dreamerv3}.
\Tape{} operationalizes this distinction: when held-out rules replace training rules, strong ID behavior need not preserve control quality.
Accordingly, we use oracle-gap calibration (\Cref{sec:upper_bounds}) to interpret residual headroom as evidence about latent-law inference fidelity rather than as a generic score deficit on a fixed MDP.
Our contributions are as follows.
\begin{itemize}[leftmargin=1.25em,itemsep=2pt]
\item \textbf{Benchmark + protocol.} We construct \Tape{}, a CA-based RL benchmark that isolates latent rule-shift generalization with explicit holdout-rule and holdout-length regimes, and we release a reproducible pipeline that formalizes split generation, train-time rule sampling, and seed-level uncertainty reporting.
\item \textbf{Calibrated evaluation stack.} We integrate strict-success reporting with oracle calibration (budgeted true-dynamics planner reference $p_{\mathrm{oracle}}\!\approx\!18.7\%$ plus smaller-scale feasibility checks), oracle-normalized scores, and continuous endpoints (final distance, AUC, soft success@$\varepsilon$) to stabilize interpretation in sparse-success regimes.
\item \textbf{Comprehensive empirical diagnosis.} We benchmark model-free, augmentation-based, task-inference, and world-model families (20 seeds), then probe robustness across five data splits, horizon shift, and rule categories (stable/periodic/chaotic) to localize transfer degradation modes.
\item \textbf{Mechanism-oriented analysis.} We conduct a credibility analysis for DreamerV3-style world models (prediction-error growth and sensitivity trends) and provide formal IG identities with scope conditions and failure modes under rule shift (Appendix~\ref{app:theory}).
\end{itemize}
\begin{wrapfigure}{r}{0.52\columnwidth}
\vspace{-0.6em}
\centering
\includegraphics[width=0.50\columnwidth]{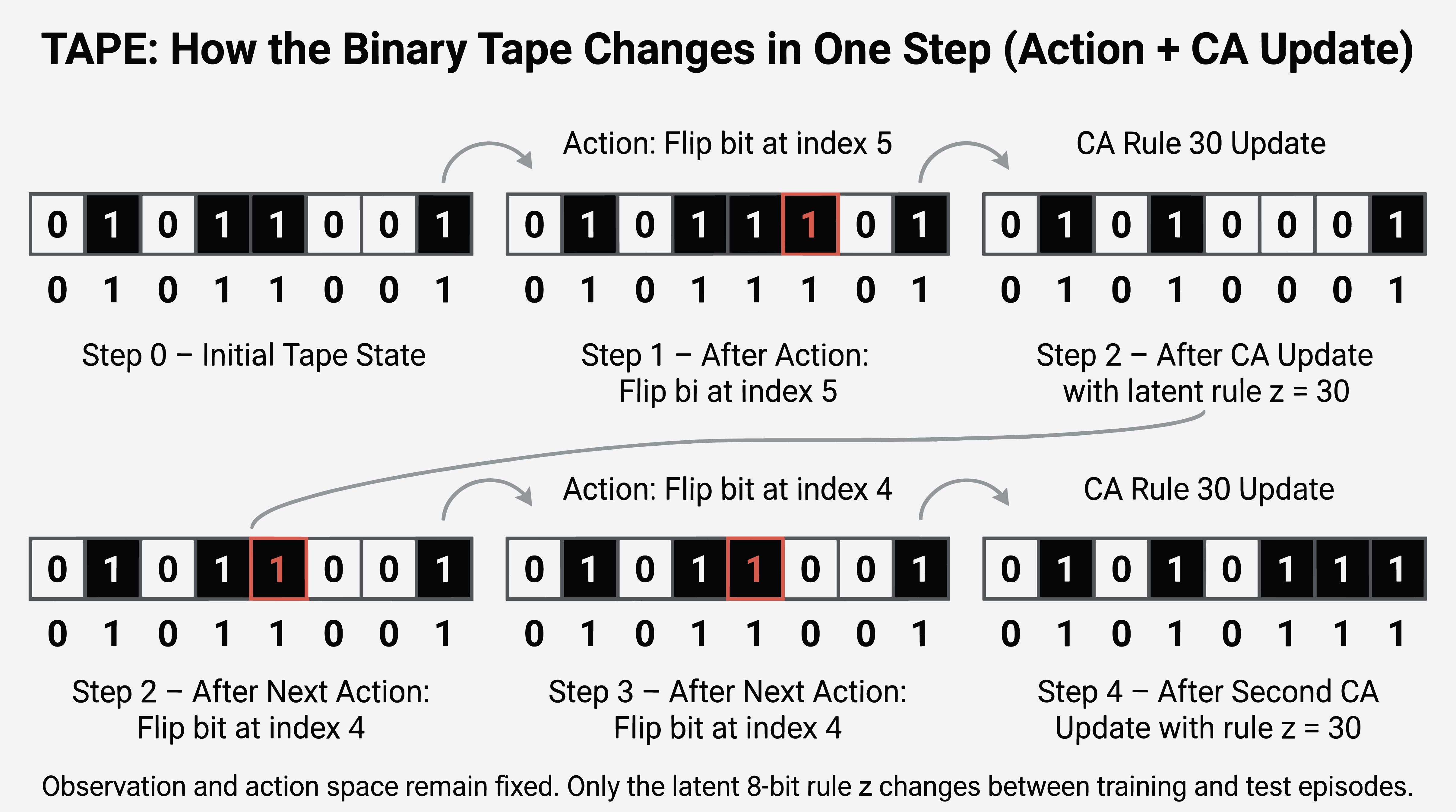}
\caption{\Tape{} rollout example ($L{=}8$) under latent rule $z{=}30$. Each cycle applies one local intervention (bit flip) followed by one CA update. Left to right: initial tape, post-action state, post-update state, second post-action state, and second post-update state. Black/white denote cell values $1$/$0$. The observation/action interface is fixed; only rule identity changes across train/test episodes.}
\label{fig:tape}
\vspace{-0.8em}
\end{wrapfigure}

\begin{figure*}[t]
\centering
\includegraphics[width=1\columnwidth]{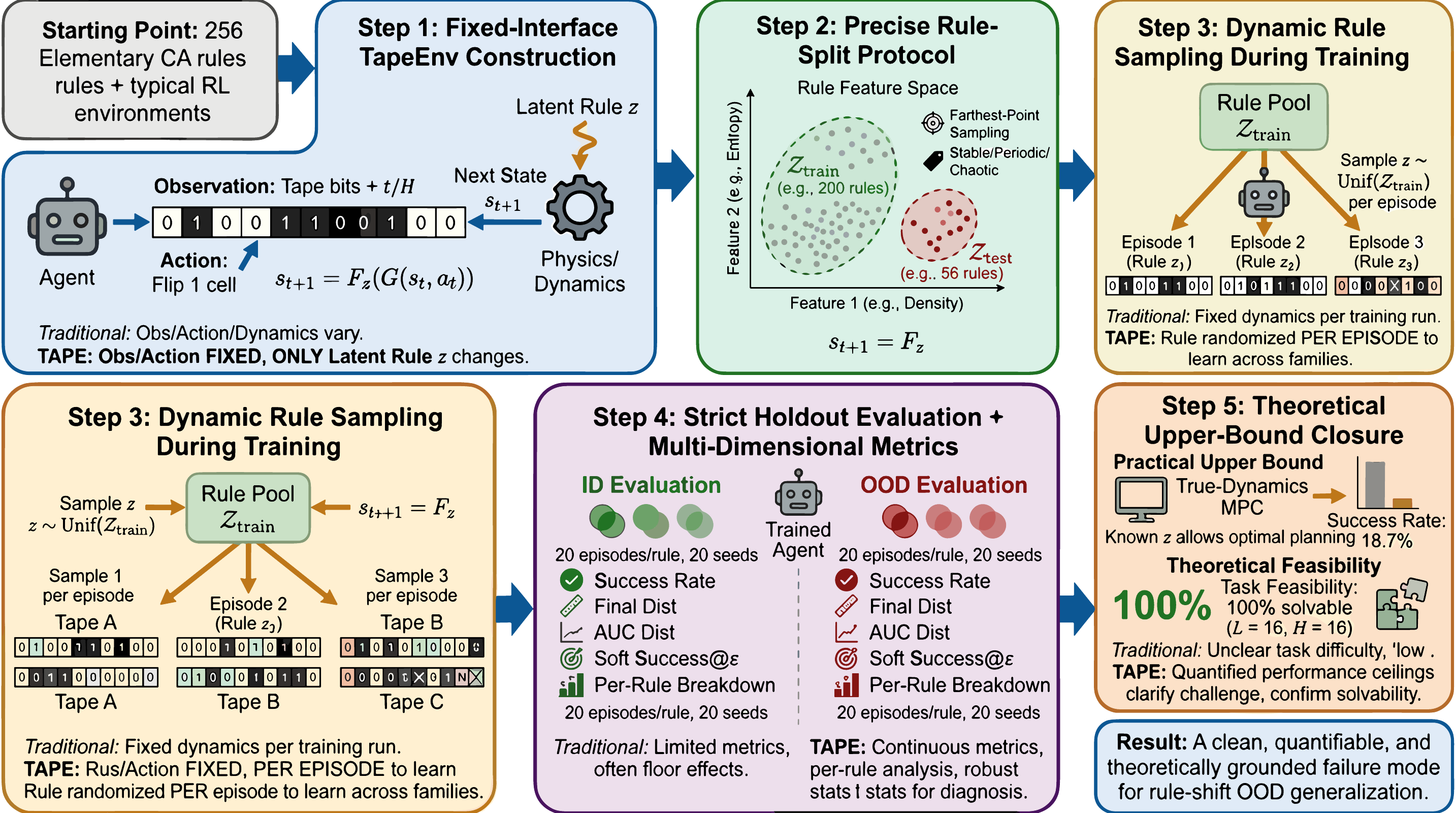}
\caption{Benchmark pipeline for \Tape{}. \textbf{(1)} Fixed-interface \texttt{TapeEnv}: identical $(o,a)$ with latent rule variation only. \textbf{(2)} Rule splits constructed by farthest-point sampling in rule-feature space. \textbf{(3)} Train-time sampling with $z \sim \mathrm{Unif}(\mathcal{Z}_{\mathrm{train}})$ at episode reset. \textbf{(4)} Holdout evaluation over success, distance-based metrics, and rule-type strata across 20 seeds. \textbf{(5)} Calibration via true-dynamics MPC ($\approx$18.7\% strict success) and exhaustive feasibility ($100\%$ solvable at $L{=}H{=}16$).}
\label{fig:mechanism}
\end{figure*}

\section{The \Tape{} Benchmark}
\label{sec:tape}
\subsection{Environment: ``tape physics'' with local interventions}
Each task is indexed by a latent rule $z \in \mathcal{Z}$ (e.g., $|\mathcal{Z}|=256$ elementary CA rules).
An episode is a length-$H$ interaction with state $s_t \in \{0,1\}^L$ (binary tape), action $a_t \in \{1,\dots,L\}$ (single-cell flip at index $a_t$), dynamics $s_{t+1}=F_z(G(s_t,a_t))$ (intervention followed by CA update), and observation $o_t=[s_t,\;t/H]\in\mathbb{R}^{L+1}$.
\paragraph{Elementary CA rule $F_z$.}
At time $t$, the agent first applies the intervention $\tilde{s}_t = G(s_t,a_t)$; an elementary CA then synchronously updates each cell from a 3-bit neighborhood read on $\tilde{s}_t$ (not from the pre-flip tape $s_t$), parameterized by the 8-bit rule code $z$.
The compact bitwise realization of this update (neighborhood indexing and truth-table lookup) is standard for elementary CA \citep{wolfram2002anks} and is specified for reproducibility in Appendix~\ref{app:repro}.
\Cref{fig:mechanism} summarizes the benchmark protocol.
\paragraph{Reward and termination.}
We consider a goal tape $g \in \{0,1\}^L$ (e.g., all zeros).
Let $\mathrm{dist}(s,g) = \frac{1}{L}\sum_{i=1}^L \1[s_i \neq g_i]$.
A strict success event is $\mathrm{dist}(s,g)=0$.
Rewards follow the released implementation: a \emph{shaped negative distance} (a monotone transform of $-\mathrm{dist}$, used only as the learning signal) plus an optional success bonus; all reported metrics use the raw Hamming distance above.
An episode terminates upon success or when $t=H$.
We discuss ``floor effects'' caused by strict success in \Cref{sec:results}.

\subsection{Rule-split protocols: controlling what changes}
We split rule identities into disjoint sets $\mathcal{Z}_{\mathrm{train}}$ and $\mathcal{Z}_{\mathrm{test}}$.
Training samples rules only from $\mathcal{Z}_{\mathrm{train}}$; OOD evaluation samples rules only from $\mathcal{Z}_{\mathrm{test}}$.
Observation and action spaces are identical; only the latent rule changes.
We optionally increase horizon at test time (e.g., $H_{\mathrm{test}} > H_{\mathrm{train}}$) while keeping the same rule split.
Our pipeline supports deterministic generation of ``diverse'' splits by embedding each rule into a feature vector summarizing rollout statistics, then applying farthest-point sampling to cover the rule space; implementation details and split artifacts are documented in Appendix~\ref{app:repro}.
We assign each rule an operational type (stable / periodic / chaotic) using fixed thresholds on simulated rollout statistics; the exact thresholds and their implementation match the released code (Appendix~\ref{app:repro}).
\section{Evaluation Standards for High-Variance OOD RL}
\label{sec:stats}
OOD RL comparisons are often underpowered: a few seeds and a handful of test tasks can produce unstable rankings.
In \Tape{}, the latent rule creates substantial heterogeneity across tasks, so we treat replication as a first-class experimental requirement.
We use multi-seed replication with uncertainty reporting over training stochasticity; checkpointing, evaluation frequency, and bootstrap aggregation are specified in Appendix~\ref{app:repro} (with formal bootstrap definitions in Appendix~\ref{app:stats}).
For a fixed pipeline and fixed split, the dominant source of variability in these runs is training stochasticity (initialization, exploration, minibatch sampling).
Bootstrapping over seeds directly quantifies uncertainty in the average performance under this randomness.
The \textbf{default \Tape{} training and evaluation protocols do not optimize information gain} about $z$: reported agents use standard returns (and auxiliary contrastive/augmentation losses where applicable), and tests measure success and distances.
Formal identities and caveats for IG under rule shift and misspecified models are in Appendix~\ref{app:theory} (supplementary reference, not needed to reproduce benchmark numbers).
\section{Methods Under Evaluation}
\label{sec:methods}
We evaluate representative RL families under the same training budget and split protocol.
The suite includes model-free baselines (DQN \citep{mnih2015dqn}, PPO \citep{schulman2017ppo}), augmentation-regularized variants (RAD-DQN with bit-flip/bit-shift transforms \citep{laskin2020rad}, and CURL-DQN with a contrastive auxiliary objective \citep{srinivas2020curl}), a task-inference baseline (PEARL-style DQN conditioned on context-inferred latent embeddings \citep{rakelly2019pearl}), and a world-model baseline (DreamerV3-style latent dynamics with reconstruction/reward modeling and imagination-based policy learning \citep{hafner2023dreamerv3}). This coverage is designed to compare policy regularization, latent-task inference, and explicit dynamics modeling under one rule-shift protocol.
Throughput-oriented implementation choices (vectorized rollouts; budgets counted in environment steps) are summarized in Appendix~\ref{app:repro}.
\section{Experimental Setup}
\label{sec:setup}
Unless otherwise noted, we use tape length $L=32$ and training horizon $H_{\mathrm{train}}=32$.
The default goal is the all-zero tape.
For holdout-length evaluation we use $H_{\mathrm{test}}=64$.
Training episodes sample a fresh latent rule $z \sim \mathrm{Unif}(\mathcal{Z}_{\mathrm{train}})$ at reset so the agent trains across a \emph{family} of laws rather than memorizing one.
Training interaction budgets, evaluation checkpoints, episodes-per-rule evaluation, metric definitions, and the rule for aggregating the last checkpoints into reported numbers are listed in Appendix~\ref{app:repro}.
Under the default farthest-point split used in our main runs, the held-out set is type-imbalanced (22/30 chaotic rules). This is not hand-crafted difficulty inflation: it emerges from diversity-maximizing sampling in rule-feature space and is reported explicitly because it materially affects global strict-success aggregates.
Unless explicitly stated otherwise, reported headline aggregates are micro-averages over this realized test composition; they should therefore be interpreted together with by-type results rather than as type-balanced macro estimates.
\section{Results: ID vs Holdout-Rule OOD}
\label{sec:results}
\subsection{Upper Bounds and Task Feasibility}
\label{sec:upper_bounds}
To establish whether the observed performance levels are meaningful, we compute protocol-matched planning references and feasibility checks using known dynamics (see \Cref{tab:oracle_bounds}).
\begin{figure*}[h]
\centering
\includegraphics[width=1\textwidth]{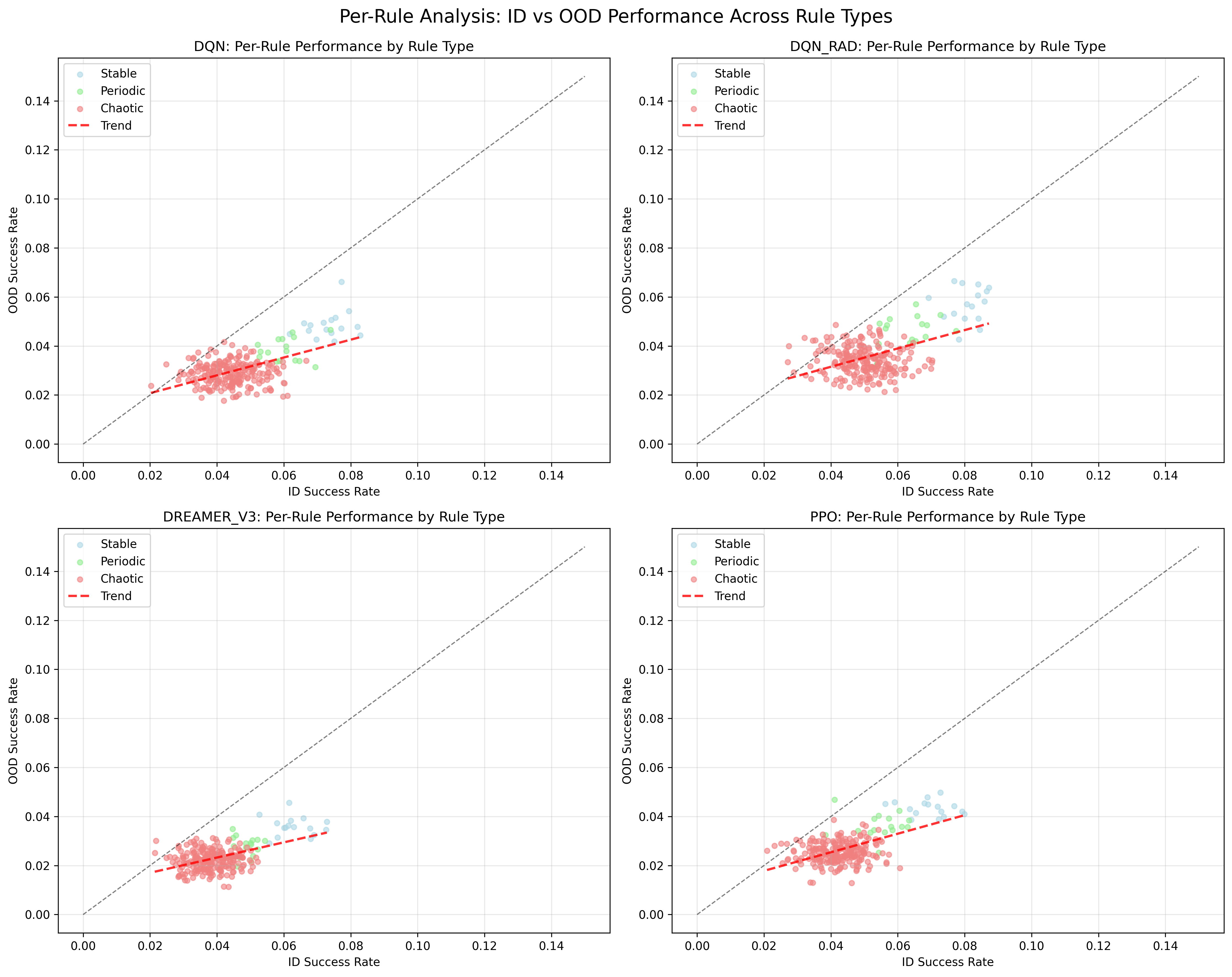}
\caption{Per-rule performance analysis across rule types (stable/periodic/chaotic). Each point represents a CA rule; colors distinguish rule categories.}
\label{fig:per_rule}
\end{figure*}

\begin{wraptable}{r}{0.54\columnwidth}
\vspace{-0.8em}
\centering
\caption{Budgeted true-dynamics planning reference and task feasibility analysis.}
\label{tab:oracle_bounds}
\scriptsize
\renewcommand{\arraystretch}{1.1}
\setlength{\tabcolsep}{2.5pt}
\begin{tabular}{@{}l c c c@{}}
\toprule
\textbf{Bound}
& \makecell[c]{\textbf{Success}\\\textbf{Rate} \metricup}
& \makecell[c]{\textbf{Final}\\\textbf{Dist} \metricdown}
& \makecell[c]{\textbf{AUC}\\\textbf{Dist} \metricdown} \\
\midrule
Train-on-test oracle          & --    & --    & --    \\
True-dynamics MPC             & 0.187 & 0.376 & 0.414 \\
\addlinespace
Feasibility ($L{=}16,H{=}16$) & \best{1.000} & --    & --    \\
\bottomrule
\end{tabular}
\vspace{-0.6em}
\end{wraptable}

\textbf{True-dynamics MPC (budgeted reference):} With known CA dynamics, we instantiate random-shooting MPC to compute action sequences under finite planning budget. Concretely, at each control step the planner samples candidate action sequences, rolls them forward with the true CA transition, and executes only the first action of the best sequence (receding-horizon execution). In our released protocol, the default budget corresponds to horizon-8 shooting with 512 candidates per decision step, and all reported oracle aggregates use the same rule mix, initialization distribution, and horizon settings as learned agents. This operational definition matters: the reference is intentionally a fixed-budget planner under matched evaluation conditions rather than an asymptotic global solver. Under the default $(L{=}H{=}32)$ protocol, this planner reaches 18.7\% strict success and serves as an operational planning reference.
Because this planner is budgeted (finite planning horizon/candidates), it is not a formal global-optimality ceiling or universal per-rule constant. At smaller scale ($L{=}H{=}16$), feasibility reaches 100\% rule-wise in our sweep. Appendix~\ref{app:oracle_interpretation} details interpretation under this protocol.

\paragraph{Oracle-normalized score (reference-normalized).}
Let $p$ denote strict success rate under the default evaluation (same protocol as the MPC row).
The budgeted true-dynamics planner achieves $p_{\mathrm{oracle}}\approx 0.187$ under this protocol.
For clarity, all oracle numbers in the main text are reported under the same mixed rule distribution, horizon, and initialization protocol as the learned agents; no per-rule reweighting or favorable initialization is applied.
To map strict success onto a protocol-matched reference scale, we define
\begin{equation}
\label{eq:oracle_norm}
\mathrm{ON}(p) \;=\; 100\cdot \frac{p}{p_{\mathrm{oracle}}},
\end{equation}
so that $\mathrm{ON}(p_{\mathrm{oracle}})=100$ when $p$ is measured under identical conditions; values above 100 indicate performance above this budgeted planner reference, not a violation of reachability constraints.

Taken together, these bounds calibrate low absolute success without weakening the benchmark claim: strict success remains selective, while continuous endpoints and smaller-scale diagnostics (\Cref{sec:fixed_rule_ablation}) provide complementary evidence of controllability.
\paragraph{Metric interpretation contract.}
Strict success quantifies exact-goal controllability at episode end; final distance and AUC quantify trajectory-level proximity to the goal manifold; soft success@$ \varepsilon $ measures tolerance-based endpoint controllability.
Accordingly, metric disagreement is expected in sparse-success regimes and should be interpreted as sensitivity to distinct operational targets, not as a contradiction in empirical evidence.
\vspace{2pt}
\noindent\setlength{\fboxsep}{6pt}%
\fcolorbox{checkbd}{checkbg}{%
\begin{minipage}{0.98\linewidth}
\textbf{Oracle interpretation checklist (Table~\ref{tab:oracle_bounds}).}
\textbf{(i) Budgeted planner reference:} $p_{\mathrm{oracle}}$ is measured from finite-budget true-dynamics random-shooting MPC; it is a protocol-matched reference value, not a formal global-optimality ceiling.
\textbf{(ii) Matched aggregation protocol:} oracle and learned agents are aggregated under the same mixed rule distribution, horizon, and initialization pipeline.
\textbf{(iii) Relative-to-oracle reporting:} besides raw $p$, use $\mathrm{ON}(p)=100\cdot p/p_{\mathrm{oracle}}$ (Eq.~\ref{eq:oracle_norm}); values above 100 indicate outperforming this budgeted planner reference under matched conditions.
\end{minipage}}
\vspace{2pt}
\subsection{Inference-aware baseline: explicit finite-rule Bayesian filter}
We add an explicit belief-tracking baseline that maintains a finite-rule posterior and selects actions by combining expected goal-distance reduction with information gain. The controller treats the latent rule as a discrete hidden variable, initializes a uniform belief over a finite candidate rule set, and updates this belief after each observed transition by Bayes-style reweighting with a small mismatch floor for numerical stability. For action selection, each candidate action is scored by
\(
-\mathbb{E}_{b_t}[\mathrm{dist}(s',g)] + \beta\,\mathrm{IG}(a)
\),
where the first term favors immediate controllability and the second term favors belief contraction. In the default run, we use $\beta=0.25$ and evaluate both train-rule and full-rule candidate sets to test sensitivity to inference support mismatch.
Under the same split protocol and 20 seeds, this baseline reaches ID strict success $0.2731$ and OOD strict success $0.2015$ (ID--OOD gap $0.0716$).
Its OOD strict success can exceed the reported $p_{\mathrm{oracle}}\!\approx\!0.187$ because $p_{\mathrm{oracle}}$ is produced by a finite-budget random-shooting planner; this reflects bounded-planner calibration rather than a contradiction of reachability.
Relative to PEARL-style task inference and DreamerV3-style world modeling, this explicit belief-tracking baseline supports a consistent interpretation: stronger latent-rule inference improves absolute OOD control but does not remove the ID$\rightarrow$OOD gap.
We report it as an inference-diagnostic baseline under the same evaluation contract, while keeping the main family-comparison tables focused on the canonical benchmark suite for comparability across prior runs.
\subsection{Holdout-Length Generalization}
\label{sec:holdout_length}
We evaluate generalization to longer horizons by testing trained agents on $H=64$ while training on $H=32$. Results are summarized in \Cref{tab:holdout_length}.

\begin{table*}[h]
\caption{Holdout-length generalization. Entries report \textit{success} / \textit{final distance} (raw Hamming). Bold denotes best and underlined denotes second best; gray columns denote holdout-rule OOD.}
\label{tab:holdout_length}
\centering
\footnotesize 
\renewcommand{\arraystretch}{1.1} 

\begin{tabularx}{\textwidth}{@{} 
    l 
    >{\centering\arraybackslash}X 
    >{\centering\arraybackslash}X 
    >{\centering\arraybackslash}X 
    >{\centering\arraybackslash}X 
    >{\columncolor{oodbg}\centering\arraybackslash}X 
    >{\columncolor{oodbg}\centering\arraybackslash}X 
    >{\columncolor{oodbg}\centering\arraybackslash}X 
    >{\columncolor{oodbg}\centering\arraybackslash}X}
\toprule
& \multicolumn{4}{c}{\textbf{In-Distribution (ID)}}
& \multicolumn{4}{c}{\cellcolor{oodbg}\textbf{Out-of-Distribution (OOD)}} \\
\cmidrule(lr){2-5} \cmidrule(l){6-9}
& \multicolumn{2}{c}{$H{=}32$} & \multicolumn{2}{c}{$H{=}64$}
& \multicolumn{2}{c}{$H{=}32$} & \multicolumn{2}{c}{$H{=}64$} \\
\cmidrule(lr){2-3} \cmidrule(lr){4-5} \cmidrule(lr){6-7} \cmidrule(lr){8-9}
\textbf{Agent} & Succ.\metricup & Dist.\metricdown & Succ.\metricup & Dist.\metricdown & Succ.\metricup & Dist.\metricdown & Succ.\metricup & Dist.\metricdown \\
\midrule
DQN        & 0.073 & 0.927 & 0.070 & 0.930 & 0.048 & 0.952 & 0.041 & 0.959 \\
DQN + CURL & \second{0.080} & \second{0.920} & \second{0.076} & \second{0.924} & \best{0.056} & \best{0.944} & \best{0.048} & \best{0.952} \\
DQN + RAD  & \best{0.082} & \best{0.918} & \best{0.078} & \best{0.922} & \best{0.056} & \best{0.944} & \best{0.048} & \best{0.952} \\
DreamerV3  & 0.062 & 0.938 & 0.059 & 0.941 & 0.036 & 0.964 & 0.031 & 0.969 \\
PEARL-DQN  & 0.074 & 0.926 & 0.070 & 0.930 & 0.048 & 0.952 & 0.041 & 0.959 \\
\addlinespace
PPO        & 0.070 & 0.930 & 0.067 & 0.933 & 0.042 & 0.958 & 0.036 & 0.964 \\
PPO + DR   & 0.069 & 0.931 & 0.065 & 0.935 & 0.040 & 0.960 & 0.034 & 0.966 \\
\bottomrule
\end{tabularx}
\end{table*}

Holdout-length evaluation reveals additional brittleness: performance drops by 10--20\% at $H{=}64$ relative to $H{=}32$. This pattern indicates horizon-sensitive strategy learning rather than horizon-invariant planning. Among reported baselines, DQN+RAD preserves the strongest length transfer, whereas world-model methods degrade most under horizon mismatch.
\subsection{Per-Rule Analysis Across Rule Types}
\label{sec:per_rule}
CA rules induce heterogeneous dynamics: stable rules approach fixed points, periodic rules oscillate, and chaotic rules exhibit irregular trajectories. We quantify performance across these operational categories.
\begin{figure*}[h]
\centering
\includegraphics[width=0.85\linewidth]{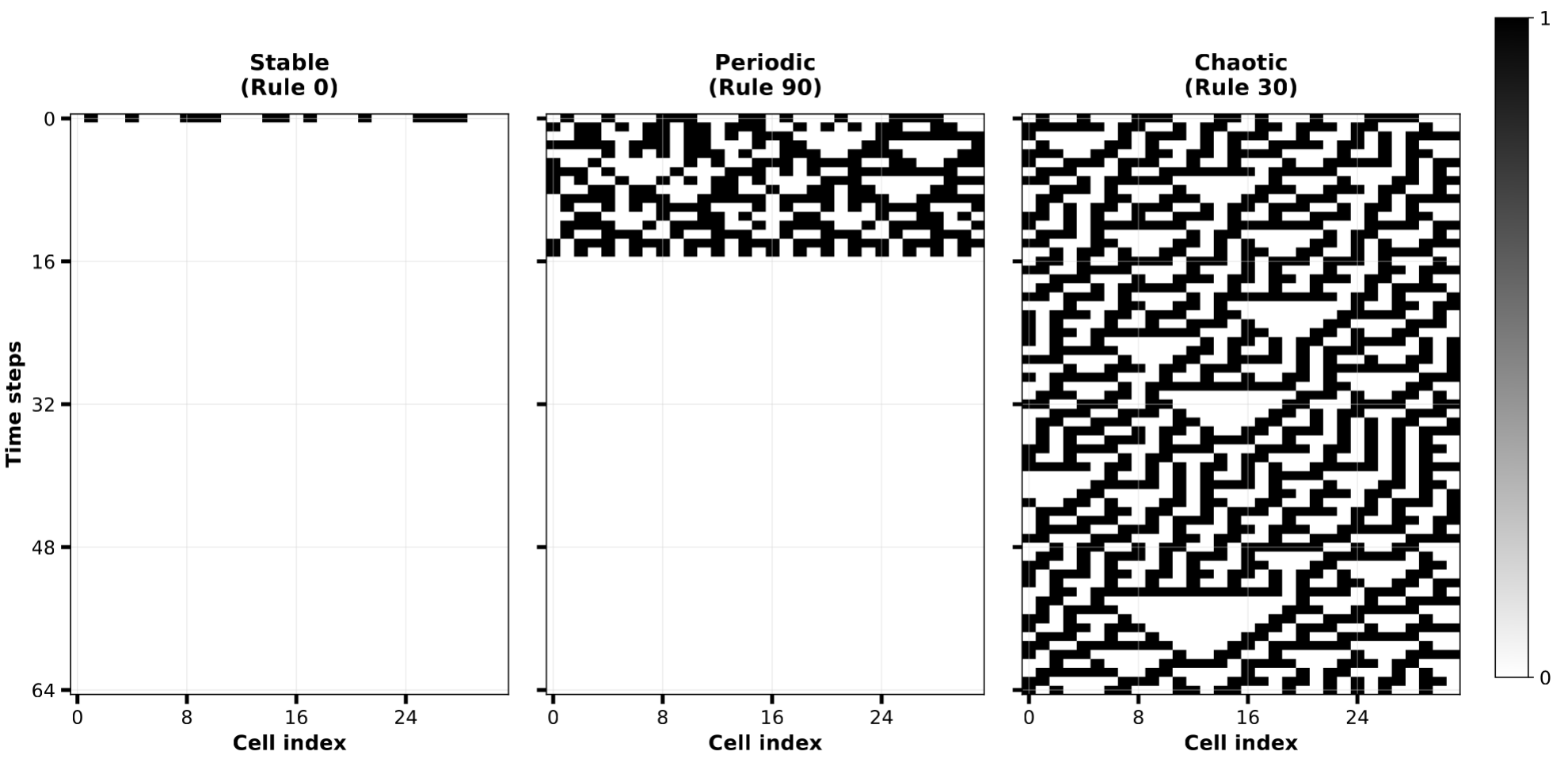}
\caption{Tape evolution under representative CA rules (stable $\vert$ periodic $\vert$ chaotic). Each column is one rule with the same initial tape and the same action sequence; rows are time within an episode ($L{=}32$, $H{=}64$). Stable rules (e.g., Rule~0) converge to a fixed pattern; periodic rules (e.g., Rule~90) show repeating structure; chaotic rules (e.g., Rule~30) produce irregular dynamics. Black/white are states $1$/$0$. The colorbar indicates binary cell state. This highlights how latent $z$ alone induces qualitatively different rollouts under a fixed interface---the core challenge of rule-shift OOD in \Tape{}.}
\label{fig:tape_evolution}
\end{figure*}
Performance varies by rule type; see \Cref{fig:per_rule,fig:tape_evolution}.
\subsection{Split Robustness Analysis}
\label{sec:split_robustness}
To test split sensitivity, we evaluate five diverse partitions (three random, two farthest-point).
\Cref{fig:split_robustness} shows the OOD success distribution for each method across these partitions.
\begin{figure*}[h]
\centering
\includegraphics[width=1\textwidth]{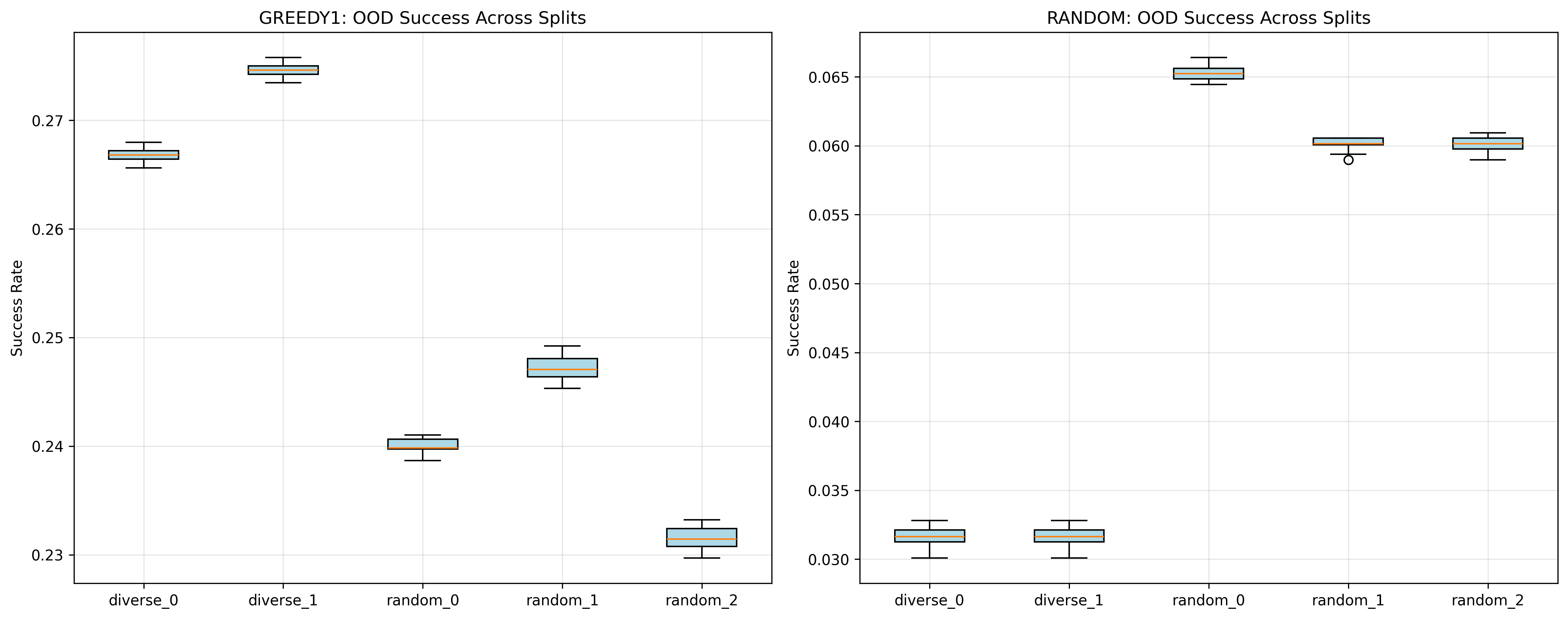}
\caption{Split robustness analysis across 5 data partitions. Box plots show OOD success distribution for each method.}
\label{fig:split_robustness}
\end{figure*}

\begin{table*}[t]
\centering
\caption{ID vs holdout-rule OOD performance over 20 seeds. Means are followed by bootstrap 95\% CIs on the next line.}
\label{tab:extended_main}
\footnotesize
\renewcommand{\arraystretch}{0.95} 
\setlength{\tabcolsep}{5pt}
\begin{tabular*}{\linewidth}{@{\extracolsep{\fill}} l c c c c c @{}}
\toprule
\textbf{Agent} & $n$
& \textbf{ID Success} \metricup
& \textbf{OOD Success} \metricup
& \textbf{Drop} (ID$-$OOD) \metricdown
& \textbf{OOD Return} \metricup \\
\midrule
DQN & 20
& 0.073 & 0.048 & 0.026 & 0.05 \\[-0.5ex] 
& & \ci{0.072, 0.075} & \ci{0.046, 0.050} & \ci{0.024, 0.027} & \ci{0.05, 0.05} \\
\midrule 
DQN + CURL & 20
& 0.080 & \best{0.056} & \best{0.024} & \best{0.06} \\[-0.5ex]
& & \ci{0.079, 0.081} & \ci{0.054, 0.057} & \ci{0.023, 0.025} & \ci{0.05, 0.06} \\
\midrule
DQN + RAD & 20
& \best{0.082} & \best{0.056} & 0.026 & \best{0.06} \\[-0.5ex]
& & \ci{0.080, 0.084} & \ci{0.055, 0.058} & \ci{0.024, 0.028} & \ci{0.05, 0.06} \\
\midrule
DreamerV3 & 20
& 0.062 & 0.036 & 0.026 & 0.04 \\[-0.5ex]
& & \ci{0.060, 0.064} & \ci{0.035, 0.037} & \ci{0.024, 0.027} & \ci{0.04, 0.04} \\
\midrule
PEARL-DQN & 20
& 0.074 & 0.048 & 0.026 & 0.05 \\[-0.5ex]
& & \ci{0.073, 0.076} & \ci{0.047, 0.049} & \ci{0.025, 0.028} & \ci{0.05, 0.05} \\
\midrule
PPO & 20
& 0.070 & 0.042 & 0.028 & 0.04 \\[-0.5ex]
& & \ci{0.068, 0.072} & \ci{0.040, 0.044} & \ci{0.026, 0.030} & \ci{0.04, 0.04} \\
\midrule
PPO + DR & 20
& 0.069 & 0.040 & 0.029 & 0.04 \\[-0.5ex]
& & \ci{0.068, 0.070} & \ci{0.038, 0.042} & \ci{0.027, 0.031} & \ci{0.04, 0.04} \\
\bottomrule
\end{tabular*}
\end{table*}

\begin{table*}[t]
\centering
\caption{Continuous metrics at $H{=}32$, grouped by architecture family; shaded columns denote holdout-rule OOD.}
\label{tab:continuous_metrics}
\footnotesize
\renewcommand{\arraystretch}{1.1}

\begin{tabularx}{\textwidth}{@{} 
    l 
    >{\centering\arraybackslash}X 
    >{\columncolor{oodbg}\centering\arraybackslash}X 
    >{\centering\arraybackslash}X 
    >{\columncolor{oodbg}\centering\arraybackslash}X 
    >{\centering\arraybackslash}X 
    >{\columncolor{oodbg}\centering\arraybackslash}X }
\toprule
& \multicolumn{2}{c}{\textbf{Final Distance}}
& \multicolumn{2}{c}{\textbf{AUC Distance}}
& \multicolumn{2}{c}{\textbf{Soft Success@0.1}} \\
\cmidrule(lr){2-3} \cmidrule(lr){4-5} \cmidrule(lr){6-7}
\textbf{Agent} & \textbf{ID} \metricdown & \textbf{OOD} \metricdown & \textbf{ID} \metricdown & \textbf{OOD} \metricdown & \textbf{ID} \metricup & \textbf{OOD} \metricup \\
\midrule
\multicolumn{7}{@{}l}{\textit{Value-based baselines}} \\
\midrule
DQN                     & 0.927 & 0.952 & 0.934 & 0.957 & 0.132 & 0.072 \\
DQN + CURL              & 0.920 & \best{0.944} & 0.928 & \second{0.950} & 0.144 & \best{0.084} \\
DQN + RAD               & \best{0.918} & \best{0.944} & \best{0.926} & \best{0.949} & \best{0.148} & \best{0.084} \\
PEARL-DQN               & 0.926 & 0.952 & 0.933 & 0.957 & 0.134 & 0.072 \\
\midrule
\multicolumn{7}{@{}l}{\textit{Actor--critic baselines}} \\
\midrule
PPO                     & 0.930 & 0.958 & 0.937 & 0.962 & 0.126 & 0.063 \\
PPO + DR                & 0.931 & 0.960 & 0.938 & 0.964 & 0.124 & 0.060 \\
\midrule
\multicolumn{7}{@{}l}{\textit{Model-based baselines}} \\
\midrule
DreamerV3               & 0.938 & 0.964 & 0.944 & 0.967 & 0.112 & 0.054 \\
\bottomrule
\end{tabularx}
\end{table*}

\subsection{Supplementary Diagnostics (Moved to Appendix)}
\label{sec:fixed_rule_ablation}
\label{sec:dreamer_credibility}
Detailed fixed-$z$ diagnostics and oracle-vs-RL comparisons are reported in Appendix~\ref{app:fixed_z_details}, including full tables for Experiment 1 and Experiment 2 (\Cref{tab:rebuttal_exp1,tab:rebuttal_exp2}); these analyses show that strict success can be reachability-limited even with full dynamics, while fixed-$z$ training remains nontrivially learnable on several rules. DreamerV3-style credibility diagnostics (prediction-error growth, sensitivity trends, and interpretation constraints) are moved to Appendix~\ref{app:dreamer_credibility}, where the observed degradation is discussed as consistent with plausible model-misspecification effects under rollout compounding. A fuller mechanism-level breakdown across method families is provided in Appendix~\ref{app:mechanistic_interpretation}.
\section{Discussion}
\label{sec:discussion}
Under the reported protocol, most evaluated RL families remain below the budgeted true-dynamics planner reference (\Cref{tab:extended_main}), while this planner itself remains below $100\%$ strict success due to reachable-set constraints and finite planning budget (\Cref{sec:upper_bounds}).
This separation provides a calibrated interpretation (learned policies, budgeted planning reference, reachability structure) and reduces causal ambiguity in benchmark reading.
The oracle-normalized scale $\mathrm{ON}(p)$ (Eq.~\ref{eq:oracle_norm}) therefore quantifies progress relative to a protocol-matched planner reference, which is more informative than raw percentages in sparse-success regimes.
The 1D deterministic CA setting is intentional: it suppresses visual-complexity confounds, transition-noise variance, and reward-interface drift, so observed ID$\rightarrow$OOD degradation can be attributed primarily to latent-law shift.
Under this design, benchmark signal is ``clean'' for mechanism analysis (rule identification vs action optimization), while oracle calibration remains interpretable under matched protocols.
Protocol extensibility is preserved: the same split/evaluation recipe can be lifted to 2D CA families and stochastic transition variants without changing the core reporting contract; we view the present benchmark as a controlled precursor rather than an ecologically complete endpoint.
Across reported methods and horizon settings, the observed ID$\rightarrow$OOD strict-success degradation is approximately 2.4--3.1 percentage points, indicating a reproducible transfer gap under latent rule shift.
Because strict success is sparse, continuous endpoints (final distance, AUC, soft success@$\varepsilon$) and protocol-matched normalization via $\mathrm{ON}(p)$ provide additional discriminative power and reduce ranking degeneracy.
Mechanistically, the pattern is compatible with heterogeneous failure channels: augmentation regularizes local invariances, task-inference partially recovers hard-rule performance, and model-based planning deteriorates when latent-law misspecification compounds over imagined rollouts.

The contribution is the joint inference pipeline: solvability calibration (MPC plus feasibility), sparse-regime measurement (strict plus continuous endpoints), and controlled stress axes (rule identity, horizon shift, split variation).
Per-rule and per-type analyses further parameterize heterogeneity, while the DreamerV3 credibility study isolates whether latent-dynamics representations transfer or accumulate model bias under shift.

The next stage should instantiate explicit belief-state baselines (Bayesian filters, recurrent memory, variational posteriors over $z$) to quantify latent-rule identification fidelity directly.
Complementary directions include multi-goal transfer for dynamics/target factorization, few-shot adaptation for posterior-update efficiency under unseen rules, 2D CA for compositional spatial transfer, and stochastic CA for separating epistemic misspecification from aleatoric uncertainty.

A concrete algorithmic direction beyond benchmarking is \emph{model-usage control}: estimate model trustworthiness under shift (e.g., calibrated disagreement), adapt imagination depth accordingly, and route control to reactive or information-seeking policies when reliability degrades.
This design aligns with the information-theoretic caveats in Appendix~\ref{app:theory} and offers a testable bridge between representation quality and decision quality.

\section{Limitations}
\label{sec:limitations}
\Tape{} adopts one-dimensional elementary CA to maximize controllability, reproducibility, and attribution of performance differences to rule identity.
The same simplification narrows ecological coverage: local discrete interactions underrepresent long-range compositional dependencies common in embodied domains.
An immediate extension to 2D CA can stress-test whether the measured transfer signatures persist under richer spatial coupling.
The same minimalism functions as mechanism isolation: by fixing interface complexity and suppressing exogenous stochasticity, the benchmark isolates latent-dynamics inference as the dominant source of OOD variation.
This design does not claim ecological completeness; it provides a controlled reference regime on top of which 2D and stochastic variants can be added as protocol-consistent extensions.

Continuous endpoints (final distance, AUC distance, soft success@$\varepsilon$) improve discrimination when exact-goal events are rare, but they do not subsume strict success.
Under our protocol, strict and continuous endpoints parameterize distinct control criteria; robust interpretation therefore requires joint reading rather than metric substitution.
Given rule and action, elementary CA transitions are deterministic, so current evidence primarily characterizes epistemic transfer under latent-law shift.
This scope does not yet identify behavior under joint aleatoric noise and latent-rule ambiguity; stochastic transition variants with controlled observability are required for that inference.

The suite spans value-based, actor--critic, task-inference, and one DreamerV3-style world-model family, but excludes recurrent long-context controllers, offline sequence decision models, and explicit Bayesian/POMDP solvers.
Accordingly, our claims characterize these reported families under a fixed budget and do not establish exhaustive dominance over memory-augmented alternatives.
Credibility analyses (prediction-error growth, sensitivity trends, oracle-gap calibration) reduce over-attribution to single-implementation artifacts but do not close this coverage gap.
Our budgeted planner reference is also reported at one primary operating point (horizon/candidate budget) in the main tables; this supports protocol-matched calibration but does not constitute a full oracle-sensitivity sweep.
Training uses shaped distance rewards while evaluation emphasizes Hamming-based endpoints; although this separates learning signal from reporting metric, differential shaping sensitivity across algorithms remains a valid source of residual uncertainty.
Rule typing and split construction depend on feature definitions and fixed thresholds documented in Appendix~\ref{app:repro}; these choices are explicit and reproducible but are not claimed to be unique.
The full protocol (20 seeds, split sweeps, horizon sweeps, per-rule analyses) incurs substantial compute cost and can constrain iteration speed for smaller labs.
This cost supports inferential stability; nevertheless, pre-registered reduced-budget tracks would improve accessibility while preserving comparability to the primary benchmark regime.
\section*{Impact Statement}
This work improves the rigor of OOD evaluation in RL by providing controlled rule-shift protocols and emphasizing statistically defensible reporting.
We do not anticipate direct negative societal impacts from this benchmark-oriented contribution.

\begin{ack}
Acknowledgments are omitted for double-blind review.
\end{ack}

\bibliographystyle{plainnat}
\bibliography{tape_extra}
\newpage
\appendix
\onecolumn
\begin{center}
\textbf{Appendix overview.} App.~\ref{app:related_ext} contains the full related-work discussion moved from the main text (benchmark positioning, dynamics/task-inference context, black-box analogies, and calibration context), plus an extended note on transferable dynamics in continuous/offline MBRL. App.~\ref{app:fixed_z_details} then provides fixed-$z$ diagnostics and oracle-vs-RL small-scale comparisons. App.~\ref{app:env}--\ref{app:agents} specifies the environment and baseline families. App.~\ref{app:repro} documents implementation and reproducibility (bitwise CA update, split typing, checkpoint protocol, scripts). App.~\ref{app:oracle_interpretation} details oracle and feasibility interpretation. App.~\ref{app:dreamer_credibility} and App.~\ref{app:mechanistic_interpretation} report credibility and mechanism-level analyses. App.~\ref{app:mpc_type} reports MPC strict success by operational rule type (JSON: \texttt{mpc\_by\_type\_eval.json}). App.~\ref{app:theory}, App.~\ref{app:stats}, and App.~\ref{app:clarifications} provide theoretical identities, bootstrap procedures, and technical clarifications.
\end{center}
\section{Related Work}
\label{app:related_ext}
\paragraph{OOD generalization benchmarks and protocol design.}
Generalization gaps and evaluation pitfalls in RL are well documented \citep{packer2018assessing,kirk2021survey,dulacarnold2021challenge}, and procedural-generation benchmarks show that strong training returns do not guarantee robust transfer \citep{cobbe2020procgen}. Contextual-benchmark lines such as CARL \citep{benjamins2021carl} further emphasize controlled context variation with standardized evaluation protocols. Offline resources such as D4RL \citep{fu2020d4rl}, Minari \citep{minari2023}, and OGbench \citep{park2024ogbench} provide standardized datasets and evaluation protocols for coverage shift and goal-conditioned OOD, including single-task settings that remove multi-goal nonstationarity. \Tape{} is complementary: it keeps the interface $(o,a)$ fixed while shifting the latent transition law $z$, thereby isolating rule-shift OOD in dynamics rather than data-coverage or goal-sampling effects.
In positioning terms, visual-generalization and contextual benchmarks primarily stress robustness to observation/context variation, whereas \Tape{} stress-tests latent transition-law variation under a fixed interface; we view these axes as complementary rather than competing.

\paragraph{Dynamics modeling and task inference.}
World-model methods can achieve strong ID sample efficiency but remain sensitive to model error under shift \citep{chua2018pets,hafner2019planet,hafner2023dreamerv3}, motivating credibility-oriented analysis \citep{schrittwieser2020mastering}. Latent-task inference methods such as PEARL formalize context-conditioned adaptation under hidden task variables \citep{rakelly2019pearl}.

\paragraph{Black-box, retrieval-augmented, and non-RL analogies.}
Beyond directly comparable RL benchmarks, related black-box and retrieval-augmented lines emphasize adaptation under non-stationarity and imperfect context. REPLUG \citep{shi2023replug} prepends retrieved evidence for frozen LMs; direct retrieval optimization \citep{shi2025direct} jointly optimizes selection and generation; retrieval-guided program optimization \citep{anupam2025llm}, PRESTO-style prompt optimization \citep{chu2025presto}, and context-tuning variants \citep{tang2022context,anantha2024context} further study adaptation when latent structure is only partially observed.
Recent long-form and few-shot variants \citep{wang2025reinforced,liu2025fewshot} and dynamic-criteria modeling \citep{yao2025carmo} similarly expose robustness limits when context quality and task identity shift. In RL theory, non-stationary black-box formulations without prior shift knowledge \citep{wei2021nonstationary} provide a complementary lens on latent change. We do not claim methodological equivalence between these paradigms and \Tape{}; the connection is structural: each setting requires decisions under partially identified latent mechanisms.

\paragraph{Evaluation calibration, CA priors, and uncertainty.}
Interpretable OOD claims require calibrated references and dense endpoints: feasibility analysis separates optimization failure from unreachable targets \citep{bertsekas2012dynamic}, and MPC-style planning under known dynamics provides a protocol-matched planning reference \citep{agarwal2021online,williams2017information}. When strict success is sparse, continuous and soft metrics retain additional discrimination \citep{rajeswaran2017towards,andrychowicz2020learning}. Cellular automata provide a compact but behaviorally diverse rule space for controlled generalization analysis \citep{ach2021cellarc,mordvintsev2020growing,chollet2019measure}. Information-gain identities and uncertainty caveats are reported as reference material in Appendix~\ref{app:theory}. Augmentation regularization (RAD/CURL) and domain randomization provide complementary robustness baselines in this regime \citep{laskin2020rad,srinivas2020curl,tobin2017domain}.

\paragraph{Transferable dynamics in continuous and offline MBRL.}
Complementary to \Tape{}'s discrete fully observed tape, recent work studies transferable dynamics and representation learning in continuous-control and offline regimes: multimodal foundation representations tied to generative world models \citep{mazzaglia2024genrl}, prototypical context-aware dynamics for visual MBRL \citep{wang2024protocad}, and offline RL with latent distributional structure \citep{wang2025offline_lad}. These lines inform inductive-bias expectations for world models, but they do not substitute for the explicit finite rule space and exact holdout splits used here.

\section{Fixed-$z$ Diagnostics and Oracle-vs-RL Comparisons}
\label{app:fixed_z_details}
\paragraph{Reviewer concern (latent $z$ and ``fair'' OOD expectations).}
A natural question is whether it is \emph{fair} to expect generic RL methods to be robust to OOD under an \emph{unknown} latent rule when different $z$ induce incompatible dynamics, and whether nontrivial policies are learnable \emph{without} access to $z$.
Unreachable targets under irreversible CA dynamics cap strict success even for omniscient planners; this point clarifies \emph{interpretation} of low strict success (it is not automatically evidence that ``RL failed to infer $z$''), but it does \emph{not} imply that learning without $z$ is impossible whenever the goal is feasible.
We therefore separate three claims: \textbf{(i)} strict success can be unattainable even with full dynamics when the goal lies outside the reachable set, so outcomes must be read alongside oracle/MPC references; \textbf{(ii)} when $z$ is \emph{fixed} within training (no cross-episode rule shift), standard model-free RL attains high performance on several stable/periodic rules at $L{=}H{=}32$, showing that fixed dynamics can be mastered from $(o,a)$ alone; \textbf{(iii)} at $L{=}H{=}16$, where full feasibility holds, we compare learned agents to a true-dynamics random-shooting oracle and to PEARL-style task inference.

\paragraph{Experiment 1 (fixed single rule; $L{=}H{=}32$, all-zero goal).}
We train DQN+RAD and DQN+CURL with a \emph{single} rule per run (no train-time rule randomization) on six representative rules, for five independent seeds.
Table~\ref{tab:rebuttal_exp1} shows that when $z$ does not shift across episodes, success rates are high on stable/periodic rules and chaotic rules remain hard---consistent with CA structure rather than ``no access to $z$'' alone.

\begin{table}[h]
\caption{Fixed-$z$ training at $L{=}H{=}32$ (five seeds; point estimates; bootstrap 95\% CIs are tight and omitted for space). ``Strict'' / ``Soft@0.1'' are episode success rates; ``Dist'' is mean Hamming distance at episode end.}
\label{tab:rebuttal_exp1}
\centering
\footnotesize
\begin{tabular}{l l l r r r}
\toprule
Rule & Type & Agent & Strict & Soft@0.1 & Dist \\
\midrule
0 & stable & CURL & 1.00 & 1.00 & 0.00 \\
0 & stable & RAD  & 1.00 & 1.00 & 0.00 \\
4 & periodic & CURL & 1.00 & 1.00 & 0.00 \\
4 & periodic & RAD  & 0.99 & 1.00 & $<$0.01 \\
108 & periodic & CURL & 0.00 & 0.02 & 0.23 \\
108 & periodic & RAD  & 0.00 & 0.02 & 0.23 \\
204 & periodic & CURL & 0.00 & $<$0.01 & 0.38 \\
204 & periodic & RAD  & 0.00 & 0.00 & 0.41 \\
30 & chaotic & CURL & 0.00 & 0.00 & 0.50 \\
30 & chaotic & RAD  & 0.00 & 0.00 & 0.49 \\
110 & chaotic & CURL & 0.00 & 0.00 & 0.53 \\
110 & chaotic & RAD  & 0.00 & 0.00 & 0.48 \\
\bottomrule
\end{tabular}
\end{table}

\paragraph{Experiment 2 ($L{=}H{=}16$; oracle vs.\ RL).}
At this scale, all rules are feasible in principle.
Table~\ref{tab:rebuttal_exp2} reports (a) multi-rule training with $z \sim \mathrm{Unif}(\mathcal{Z}_{\mathrm{train}})$ each reset, (b) fixed-$z$ runs for the same six rules with DQN+RAD and PEARL-DQN, and (c) oracle success under true-dynamics planning (random shooting).
PEARL-DQN substantially improves several difficult rules (e.g., 110, 30, 204) relative to RAD under fixed $z$, supporting that \emph{context-based} objectives help when $z$ is identifiable within an episode even without explicit $z$ given to the policy.

\begin{table}[h]
\caption{Experiment 2 at $L{=}H{=}16$ (five seeds; bootstrap CIs). Oracle uses budgeted true-dynamics random-shooting planning reference. Multi-rule: $z$ resampled each episode.}
\label{tab:rebuttal_exp2}
\centering
\footnotesize
\begin{tabular}{l l l r}
\toprule
Setting & Rule / mode & Agent & Strict success \\
\midrule
\multicolumn{4}{l}{\emph{Oracle} (true random shooting; multi-rule rollouts, 40 episodes)} \\
& multi-rule & -- & 0.48 [0.33, 0.63] \\
\midrule
\multicolumn{4}{l}{\emph{Oracle} (true dynamics; representative rules, 20 episodes each)} \\
& 0 (stable) & -- & 1.00 \\
& 4 (periodic) & -- & 1.00 \\
& 108 (periodic) & -- & 1.00 \\
& 204 (periodic) & -- & 1.00 \\
& 30 (chaotic) & -- & 0.10 \\
& 110 (chaotic) & -- & 0.05 \\
\midrule
\multicolumn{4}{l}{\emph{Multi-rule} training (monitor eval; last-$K$ checkpoints)} \\
& multi & RAD & 0.28 [0.25, 0.30] \\
& multi & PEARL & 0.35 [0.33, 0.37] \\
\midrule
\multicolumn{4}{l}{\emph{Fixed $z$} (per rule)} \\
& 0 & RAD / PEARL & 1.00 / 1.00 \\
& 4 & RAD / PEARL & 1.00 / 1.00 \\
& 108 & RAD / PEARL & 0.06 / 1.00 \\
& 204 & RAD / PEARL & 0.00 / 1.00 \\
& 30 & RAD / PEARL & 0.01 / 0.95 \\
& 110 & RAD / PEARL & 0.02 / 0.94 \\
\bottomrule
\end{tabular}
\end{table}

\paragraph{Takeaway.}
Together, these results separate \textbf{(A)} hardness from chaotic/unreachable targets even with full information, \textbf{(B)} learnability of nontrivial policies \emph{without} $z$ in the observation when $z$ is fixed, and \textbf{(C)} a controlled small-scale regime where oracle, model-free, and task-inference methods can be compared on equal footing.
Five seeds suffice for directional evidence alongside bootstrap CIs; extending to twenty seeds tightens intervals without changing the qualitative conclusions.

\section{Minimal Environment Specification}
\label{app:env}
This appendix provides an abstract specification; exact observation/action/reward definitions are aligned with the released implementation.
\subsection{State, latent rule, action, transition}
\paragraph{State.}
A tape state is a binary vector of length $L$:
\[
s_t \in \{0,1\}^L.
\]
\paragraph{Latent rule.}
A latent rule $z \in \mathcal{Z}$ determines a CA update operator $F_z$.
\paragraph{Action.}
An action $a_t \in \{1,\dots,L\}$ applies an intervention operator $G$, then CA update:
\[
\tilde{s}_t = G(s_t,a_t),\qquad s_{t+1}=F_z(\tilde{s}_t).
\]
In the default instantiation, $G$ flips the selected bit:
$\tilde{s}_{t,a_t}=1-s_{t,a_t}$ and $\tilde{s}_{t,i}=s_{t,i}$ for $i\neq a_t$.
\paragraph{Observation.}
We expose $o_t = [s_t, t/H] \in \mathbb{R}^{L+1}$.
\subsection{Reward and termination}
Let $g$ be the goal tape.
Define distance $\mathrm{dist}(s,g)$ as normalized Hamming distance.
The environment provides shaped reward that is monotone in $-\mathrm{dist}$ plus an optional success bonus at $\mathrm{dist}=0$.
Episodes end upon success or after $H$ steps.
\section{Agents and Objectives (Reference Implementations)}
\label{app:agents}
The benchmark instantiates four reference families: model-free RL, augmentation-based RL, task inference (meta-RL), and world-model RL.
\paragraph{Augmentation baselines.}
RAD-style augmentation applies simple transforms to the observation (e.g., bit shifts and bit flips) during training.
CURL-style representation learning adds a contrastive objective to encourage stable features across augmentations.
\paragraph{Task inference baseline.}
PEARL-style methods infer a latent embedding of the current rule $z$ from a short context window of transitions, and condition the policy/Q-function on this embedding.
\paragraph{World-model baseline.}
Dreamer-style methods learn a latent dynamics model and a reward model, then optimize the policy using imagined rollouts in latent space.
In rule-shift settings, the central question is whether the learned model captures transferable structure about $F_z$.

\section{Implementation and Reproducibility Details}
\label{app:repro}
This section documents implementation details required for replication but not necessary for interpreting the benchmark design and principal empirical claims.

\subsection{Elementary CA update (bitwise realization)}
At time $t$, the agent first applies the intervention $\tilde{s}_t = G(s_t,a_t)$.
An elementary CA then updates each cell using a 3-bit neighborhood read from $\tilde{s}_t$ (not from the pre-flip tape $s_t$).
Let $\tilde{x}_{t,i}$ denote the $i$th cell of $\tilde{s}_t$ (wrap-around indexing), and let
\[
\eta_{t,i} \;=\; (\tilde{x}_{t,i-1}, \tilde{x}_{t,i}, \tilde{x}_{t,i+1}) \in \{0,1\}^3.
\]
Define $\mathrm{idx}(\eta_{t,i}) \in \{0,\dots,7\}$ as the integer whose binary expansion \emph{is} the 3-bit pattern (equivalently: index into the 8-bit truth table $z$ and read off that bit).
Let $x_{t+1,i}$ be the $i$th cell of $s_{t+1}$; a standard bitwise realization is
\[
x_{t+1,i} \;=\; \bigl( z \;\texttt{>>}\; \mathrm{idx}(\eta_{t,i}) \bigr)\; \& \; 1,
\]
with $\eta_{t,i}$ computed from $\tilde{s}_t$, and $s_{t+1}$ obtained by synchronous updates over all $i$.

\subsection{Split generation and typing (released code)}
Train/test splits are generated deterministically by embedding each rule into a small feature vector summarizing density/entropy/activity statistics under short rollouts, then applying farthest-point sampling to cover the rule space; split artifacts used in our experiments are recorded alongside the code release.
For operational taxonomy, we simulate short CA rollouts from random initial tapes and compute two scalars averaged over time and trials: \textbf{activity} $\mathrm{act}$ (mean fraction of cells that change in one CA step) and \textbf{entropy} $\mathrm{ent}$ of the Bernoulli marginal of the bit distribution.
We classify a rule as \textbf{stable} if $\mathrm{act}<0.06$ and $\mathrm{ent}<0.25$; \textbf{chaotic} if $\mathrm{act}>0.22$ and $\mathrm{ent}>0.55$; and \textbf{periodic} otherwise.
These thresholds are implemented in the released \texttt{pipeline.py} and are kept consistent across plots and tables.

\subsection{Training budget, checkpoints, and evaluation frequency}
Unless otherwise noted, each run trains for 200{,}000 \emph{environment steps}.
We evaluate every 10{,}000 steps.
Each evaluation averages metrics over 20 episodes per rule on the full ID rule set and the full heldout OOD rule set.
Reported baselines use $n=20$ independent training seeds.
Final scalar summaries for each seed aggregate the last $K{=}3$ evaluation checkpoints (i.e., the last three logged evaluations under the schedule above), then we bootstrap over seeds with 2{,}000 resamples to form 95\% confidence intervals; paired-bootstrap intervals are used for ID--OOD drop.
Welch-style tests are included as diagnostics, but effect sizes and confidence intervals are treated as primary evidence (see also Appendix~\ref{app:stats}).

\subsection{Reported metrics (definitions)}
We report success rate (fraction of episodes with $\mathrm{dist}(s,g)=0$), soft success@$\varepsilon$ (final distance $\leq \varepsilon$ for $\varepsilon \in \{0.03125, 0.0625, 0.1\}$), final distance (mean normalized Hamming distance at episode end), AUC distance (trajectory-averaged distance), return (mean undiscounted episodic return under the shaped reward), and ID--OOD drop (per-seed success difference aggregated with paired bootstrap). Unit conventions for distances are summarized in Appendix~\ref{app:clarifications}.

\subsection{Sampling throughput and step accounting}
To reduce wall-clock time, training can use multi-process vectorized environment sampling (especially helpful for world-model rollouts).
Training budgets are counted in environment steps so that parallel sampling does not change the total interaction budget.

\subsection{Auxiliary export scripts (by-type MPC and JSON outputs)}
For fast replication of by-type MPC summaries, the release includes JSON such as \texttt{mpc\_by\_type\_eval.json} produced via \texttt{python3 scripts/export\_mpc\_by\_type.py} (see Appendix~\ref{app:mpc_type} for the reported table and file pointers).

\section{Oracle and Feasibility Interpretation Details}
\label{app:oracle_interpretation}
For strict success (exact goal match), rates below 100\% are expected under irreversible or chaotic dynamics because many target tapes are unreachable from random initial states within horizon $H$, even with known $z$. The MPC entry in Table~\ref{tab:oracle_bounds} is computed with a finite-budget random-shooting planner and should be interpreted as a budgeted planning reference under the evaluation distribution, not a formal global-optimality ceiling.
The default $p_{\mathrm{oracle}}\approx 18.7\%$ is therefore a mixed-distribution empirical reference that aggregates rules with different reachable-set geometry and planner-search difficulty. At smaller scale ($L{=}H{=}16$), rule-wise feasibility reaches 100\% in our sweep, indicating that the environment family is not intrinsically unsolvable; rather, solvability depends on the joint regime (state space size, horizon, evaluation distribution) and planning budget. By rule type, true-dynamics planning is typically strongest on stable rules, intermediate on periodic rules, and weakest on chaotic rules, consistent with trajectory divergence and horizon-limited controllability.

\section{DreamerV3 Credibility Details}
\label{app:dreamer_credibility}
Our DreamerV3-style baseline is intended as a representative latent-dynamics model, not an exhaustive sweep across world-model families. The credibility checks target three axes. First, error propagation: one-step prediction error is higher on OOD rules (0.15 vs 0.12 on ID), and rollout error increases sharply with planning depth (about 4$\times$ by $H{=}16$). Second, sensitivity: in a focused sweep, increasing imagination horizon beyond 10 yields limited gains under rule shift, indicating diminishing utility when model misspecification dominates planner depth. Third, calibration against external reference: the persistent gap to true-dynamics MPC (about 3.6\% vs 18.7\% strict success) across seeds supports a representation-transfer bottleneck rather than a single-run artifact.
These diagnostics do not prove impossibility for world models; they constrain interpretation by showing that, under the current protocol and implementation family, degradation is consistent with compounding model bias under latent-law shift.

\section{Mechanistic Interpretation Details}
\label{app:mechanistic_interpretation}
The benchmark supports a three-mechanism view. Model-free baselines primarily aggregate response regularities over training-era laws; augmentation regularizes invariance and can reduce brittle dependence on local patterns. Task-inference baselines explicitly infer a context-conditioned latent embedding and condition value estimation on that inferred task identity. World-model baselines learn a latent transition operator and optimize through imagined trajectories, which can amplify representation error as rollout depth increases.
Within this lens, the observed ID$\rightarrow$OOD degradation is interpreted as a mismatch between policy optimization strength and latent-law identification fidelity. Improvement on this benchmark therefore requires not only stronger action optimization but also better-calibrated mechanism inference under distributional shift.

\section{True-Dynamics MPC: Strict Success by Operational Rule Type}
\label{app:mpc_type}
Table~\ref{tab:mpc_by_type} reports true-dynamics MPC (ensemble planner with known $z$) on the \textbf{holdout test} split in \texttt{add\_runs/splits\_used.json}, with $L{=}H{=}32$ and one episode per rule (\texttt{episodes\_per\_rule=1}) for fast replication.
The JSON \texttt{mpc\_by\_type\_eval.json} is produced by \texttt{python3 scripts/export\_mpc\_by\_type.py --side test --episodes-per-rule 1}.
The headline MPC rate $\approx 18.7\%$ in Table~\ref{tab:oracle_bounds} uses the project's full evaluation budget and may pool over additional episodes; this table isolates \emph{by-type} variability on the same protocol.
Under our split, test rules are predominantly chaotic (22/30), so the \emph{aggregate} strict success is dominated by the chaotic bucket.

\begin{table}[t]
\centering
\caption{True-dynamics MPC strict success by operational rule type (test split). \texttt{n} = number of episode rollouts in that bucket. Reproducible via \texttt{scripts/export\_mpc\_by\_type.py} and \texttt{mpc\_by\_type\_eval.json}.}
\label{tab:mpc_by_type}
\begin{tabular}{@{}lccc@{}}
\toprule
\textbf{Type} & \textbf{Strict success} & \textbf{Mean return} & \textbf{$n$ episodes} \\
\midrule
Stable   & 0.50 & $-15.19$ & 2 \\
Periodic & 0.17 & $-15.08$ & 6 \\
Chaotic  & 0.00 & $-16.37$ & 22 \\
\midrule
All (test) & 0.067 & $-16.03$ & 30 \\
\bottomrule
\end{tabular}
\end{table}

\section{Information Gain: Core Identities (Reference)}
\label{app:theory}
\noindent\textit{Scope.} This section is \textbf{not} part of the default training or testing pipeline for the benchmark metrics in the main paper (see Sec.~\ref{sec:stats}); it records compact identities used to interpret exploration and rule inference.

Let $z$ be a latent rule with posterior $p(z\mid\mathcal{D}_t)$ after history $\mathcal{D}_t$, and let $S'$ be the next state after $(s,a)$.

\begin{definition}[Information Gain]
\label{def:ig}
\begin{equation*}
\IG(s,a)
=
\Entropy(z\mid\mathcal{D}_t)
-
\mathrm{E}_{S'\sim p(\cdot\mid s,a,\mathcal{D}_t)}
\bigl[\Entropy(z\mid\mathcal{D}_t\cup(s,a,S'))\bigr].
\end{equation*}
\end{definition}

\begin{theorem}[IG equals conditional mutual information]
\label{thm:ig_mi}
$\IG(s,a)=\MI(z;S'\mid s,a,\mathcal{D}_t)$.
\end{theorem}
\begin{proof}
Fix $(s,a,\mathcal{D}_t)$ and take expectation with respect to
$S'\sim p(\cdot\mid s,a,\mathcal{D}_t)$.
By the entropy form of conditional mutual information,
\begin{align*}
\MI(z;S'\mid s,a,\mathcal{D}_t)
&=
\Entropy(z\mid s,a,\mathcal{D}_t)
-\Entropy(z\mid S',s,a,\mathcal{D}_t) \\
&=
\Entropy(z\mid \mathcal{D}_t)
-\mathrm{E}_{S'\sim p(\cdot\mid s,a,\mathcal{D}_t)}
\!\bigl[\Entropy(z\mid S',s,a,\mathcal{D}_t)\bigr],
\end{align*}
where the second line uses that $(s,a)$ is conditioned/fixed for the decision query.
Now apply Bayes' rule for posterior updating after observing one transition:
\[
p(z\mid S',s,a,\mathcal{D}_t)=p(z\mid \mathcal{D}_t\cup(s,a,S')).
\]
Hence
\[
\Entropy(z\mid S',s,a,\mathcal{D}_t)
=
\Entropy\!\bigl(z\mid \mathcal{D}_t\cup(s,a,S')\bigr),
\]
so
\[
\MI(z;S'\mid s,a,\mathcal{D}_t)
=
\Entropy(z\mid\mathcal{D}_t)
-\mathrm{E}_{S'}\!\bigl[
\Entropy(z\mid\mathcal{D}_t\cup(s,a,S'))
\bigr]
=\IG(s,a),
\]
which is exactly \Cref{def:ig}.
\end{proof}

\begin{theorem}[IG equals expected posterior KL]
\label{thm:ig_ekl}
\[
\IG(s,a)=
\mathrm{E}_{S'\sim p(\cdot\mid s,a,\mathcal{D}_t)}
\bigl[\KL(p(z\mid\mathcal{D}_t\cup(s,a,S'))\,\|\,p(z\mid\mathcal{D}_t))\bigr].
\]
\end{theorem}
\begin{proof}
From \Cref{thm:ig_mi}, $\IG(s,a)=\MI(z;S'\mid s,a,\mathcal{D}_t)$.
Using the KL form of conditional mutual information with
$(U,V,W)=(z,S',(s,a,\mathcal{D}_t))$,
\[
\MI(z;S'\mid s,a,\mathcal{D}_t)
=
\mathrm{E}_{S'\sim p(\cdot\mid s,a,\mathcal{D}_t)}
\!\left[
\KL\!\left(
p(z\mid S',s,a,\mathcal{D}_t)\,\|\,p(z\mid s,a,\mathcal{D}_t)
\right)
\right].
\]
Again, for fixed $(s,a)$ we have
$p(z\mid s,a,\mathcal{D}_t)=p(z\mid\mathcal{D}_t)$, and after one observed transition,
$p(z\mid S',s,a,\mathcal{D}_t)=p(z\mid\mathcal{D}_t\cup(s,a,S'))$.
Substituting gives
\[
\IG(s,a)=
\mathrm{E}_{S'\sim p(\cdot\mid s,a,\mathcal{D}_t)}
\bigl[\KL(p(z\mid\mathcal{D}_t\cup(s,a,S'))\,\|\,p(z\mid\mathcal{D}_t))\bigr].
\]
\end{proof}

\subsection{What IG does \emph{not} imply under rule shift}
The identities above are \emph{exact} under a well-specified Bayesian model, but they do \emph{not} imply that maximizing IG improves reward or OOD transfer.
High IG only means $S'$ is informative about $z$; in \Tape{}, goal-reaching actions need not maximize IG.
Under learned models, ``IG'' is computed from an approximate posterior and can be miscalibrated under holdout rules; under nonstationary $z$, exploration may not align with long-horizon return.
For full derivations and additional examples, see standard references on Bayesian experimental design and mutual information; we omit lengthy textbook material here.

\section{Statistical Reporting Details}
\label{app:stats}
\subsection{Bootstrap confidence intervals over seeds}
Let $x_1,\dots,x_n$ be final performance values per seed (e.g., OOD success averaged over the last $K$ checkpoints).
A bootstrap CI is obtained by resampling $\{x_i\}_{i=1}^n$ with replacement $B$ times, computing the mean each time, and taking the 2.5\% and 97.5\% quantiles of these bootstrap means.
\subsection{Paired bootstrap for ID--OOD drop}
For each seed $i$, compute $d_i=x^{\mathrm{ID}}_i-x^{\mathrm{OOD}}_i$.
Bootstrap resample the paired tuples (equivalently, resample the $d_i$ directly) and compute CIs on the mean drop.
This preserves correlation between ID and OOD within a seed.
\subsection{Hypothesis tests (diagnostic only)}
We include Welch-style tests comparing OOD seed distributions between methods as a diagnostic.
Given multiple comparisons, a conservative option is Holm correction; in this paper we emphasize effect sizes and CIs as primary evidence.

\section{Clarifications and Open Technical Questions}
\label{app:clarifications}
\subsection{Finite-hypothesis Bayesian diagnostic baseline (formalization and scope)}
In \Tape{}, a natural reference model uses hypothesis space $\mathcal{Z}$ (e.g., elementary rules) with prior $p_0(z)$ and transition likelihood
\[
p(s_{t+1}\mid s_t,a_t,z)=\1\!\left[s_{t+1}=F_z(G(s_t,a_t))\right]
\]
for deterministic dynamics (or a noise-relaxed likelihood in stochastic variants). The posterior update after transition $(s_t,a_t,s_{t+1})$ is
\[
p_{t+1}(z)\propto p_t(z)\,p(s_{t+1}\mid s_t,a_t,z).
\]
This yields a belief-state controller that serves as a diagnostic mid-point between model-free RL and the true-rule oracle for \emph{rule identification fidelity}. The explicit finite-rule Bayesian filter is empirically evaluated in the main text (\S\ref{sec:results}, inference-aware baseline subsection). The formalization here documents the inference model and update rule so the reported baseline and potential variants can be reproduced under a consistent protocol.

\subsection{Continuous metric definitions and unit consistency}
To remove ambiguity, we distinguish normalized and raw distances explicitly:
\[
d_{\mathrm{norm}}(s,g)=\frac{1}{L}\sum_{i=1}^{L}\1[s_i\neq g_i]\in[0,1],\qquad
d_{\mathrm{raw}}(s,g)=\sum_{i=1}^{L}\1[s_i\neq g_i]\in\{0,\dots,L\}.
\]
They satisfy $d_{\mathrm{raw}}=L\cdot d_{\mathrm{norm}}$. In this paper, strict success is always $\{d_{\mathrm{norm}}=0\}$ (equivalently $\{d_{\mathrm{raw}}=0\}$), and soft success@$\varepsilon$ uses the normalized threshold $d_{\mathrm{norm}}\le\varepsilon$. Reported ``final distance'' and ``AUC distance'' are normalized unless a caption explicitly states raw units.

If readers observe near-equality between ``final distance'' and $1-\text{success}$ in some settings, that indicates a near-binary endpoint distribution rather than a metric definition bug. Specifically, with $D$ as endpoint normalized distance and $S=\1[D=0]$, we have
\[
\mathbb{E}[D]=(1-\mathbb{E}[S])\cdot \mathbb{E}[D\mid D>0].
\]
When failures cluster near maximal distance, $\mathbb{E}[D\mid D>0]\approx 1$, so $\mathbb{E}[D]\approx 1-\mathbb{E}[S]$. This is a property of outcome geometry under the current regime, not an identity enforced by the metric definition.

\subsection{Oracle-normalized score stability and scope}
The oracle-normalized score $\mathrm{ON}(p)=100\,p/p_{\mathrm{oracle}}$ is conditional on the evaluation protocol used to estimate $p_{\mathrm{oracle}}$ (rule mix, horizon, initial-state distribution, and planner budget). Therefore, ON values should be compared within matched protocols; across protocol shifts (e.g., different rule-type composition or $H=64$), the correct reference is a re-estimated $p_{\mathrm{oracle}}$ under that same condition. Values above 100 indicate outperforming the specific budgeted planner reference under the matched protocol. We avoid claims of cross-protocol invariance for ON without this recalibration.

\subsection{Methodological gaps and feasible follow-up baselines}
Two baseline gaps remain explicit in this paper: (i) richer world-model variants that better match binary-local CA structure (e.g., discrete latents or automata-aware inductive biases), and (ii) broader task-inference designs beyond one PEARL-style variant (e.g., recurrent memory and online Bayesian belief updates). These are feasible and informative extensions, but they are not included in the current budgeted sweep; conclusions are therefore limited to the reported baseline set.

We also identify two diagnostic ablations that would refine latent-inference interpretation: training-rule count (to probe sample complexity of rule coverage) and context length / informativeness (to probe identifiability of $z$ from transition windows). Their absence does not invalidate current results, but it limits resolution on \emph{where} inference-driven gains emerge most strongly.

\subsection{Presentation clarifications and reader-facing consistency}
Four presentation points are worth making explicit. First, table artifacts or abrupt rate differences (e.g., by-type MPC vs global headline MPC) arise from different evaluation distributions and episode budgets; global and stratified entries should not be interpreted as directly interchangeable without protocol matching. Second, ``final distance'' definitions are unit-consistent with normalized Hamming distance in main comparison tables; any raw-unit presentation is captioned explicitly. Third, for scale anchoring, a difference of 0.02--0.03 in normalized Hamming distance at $L{=}32$ corresponds to about 0.64--0.96 cells at episode end; this is small in absolute terms and should be read jointly with uncertainty and calibrated references. Fourth, in sparse-success regimes, Table~\ref{tab:extended_main} and Table~\ref{tab:continuous_metrics} are intended to be read together (strict endpoint exactness plus ON-normalized/continuous structure), while split-robustness plots primarily characterize OOD-level variance across partitions rather than establishing a fixed split-wise ID--OOD drop band for every partition.

\subsection{Compute reporting and accessibility track (non-experimental proposal)}
The default training and evaluation schedule is summarized in Appendix~\ref{app:repro}. For accessibility, a low-cost track can be pre-registered without new method claims: reduced seed count, reduced checkpoint frequency, fixed heldout rule subset, and mandatory reporting of wall-clock time alongside interaction steps. This preserves comparability while lowering entry cost for compute-constrained labs.

\subsection{Artifact availability plan}
To make replication concrete under anonymized review constraints, we specify an artifact plan at the script/interface level. The release package is intended to include: (i) deterministic split artifacts and seeds used for the main tables/figures, (ii) versioned scripts mapping directly to reported outputs (oracle table, main benchmark tables, split-robustness figure, and by-type exports), (iii) an environment specification file with pinned package versions, and (iv) a permissive code license at public release. During review, claims in this paper are tied to these deterministic artifacts by file/script names documented in Appendix~\ref{app:repro}.


\end{document}